\documentclass{llncs}

\usepackage[letterpaper,margin=1.4in]{geometry}
\usepackage{amsmath}
\usepackage{amssymb}
\usepackage{epsfig}
\usepackage{times}

\newtheorem{fact}[theorem]{Fact}

\newcommand{\nodes}{\mathit{nodes}}
\newcommand{\edges}{\mathit{edges}}

\newcommand{\HG}{{\cal H}}
\newcommand{\JT}{J\!T}

\newcommand{\vars}{\mathit{var}}

\newcommand{\A}{\mathcal{A}}
\newcommand{\B}{\mathcal{B}}
\newcommand{\V}{\mathcal{V}}

\newcommand{\size}[1]{\lVert #1 \rVert}
\newcommand{\Pol}{\mbox{\rm P}}
\newcommand{\NP}{\mbox{\rm NP}}

\newcommand{\FPT}{\mbox{\rm FPT}}

\newcommand{\tuple}[1]{\langle#1\rangle}
\newcommand{\nop}[1]{}

\newcommand{\longv}[1]{}

\newcommand{\tpCovered}{\mbox{\it tp-covered}}

\newcommand{\CSP}{\mbox{\rm CSP}}
\newcommand{\ECSP}{\mbox{\rm ECSP}}

\renewcommand{\A}{\mathbb{A}}
\renewcommand{\B}{\mathbb{B}}

\newcommand{\alg}{$\mathtt{ComputeAllSolutions}_{\scriptsize \mbox{\tt DM}}$}
\newcommand{\findCertificate}{$\mathtt{ComputeCertifiedSolutions}_{\scriptsize \mbox{\tt DM}}$}
\newcommand{\bm}[1]{{\bf #1}}
\newcommand{\boA}{\bm{A}}

\newcommand{\all}{\bm{-}}
\newcommand{\lDM}{\ell\mbox{-\tt DM}}
\newcommand{\rDM}{r\mbox{-\tt DM}}

\newcommand{\DRV}{\mathit{drv}}
\newcommand{\dom}{\mathit{dom}}

\newcommand{\GAC}{{\tt GAC}}

\begin{document}

\title{On The Power of Tree Projections:\\
Structural Tractability of Enumerating CSP Solutions}

\author{Gianluigi Greco\inst{1}\and Francesco Scarcello\inst{2}}

\institute{
  Dept. of Mathematics\inst{1} and DEIS\inst{2}, University of Calabria, 87036, Rende, Italy\\
  {\tt ggreco@mat.unical.it}, {\tt scarcello@deis.unical.it}\\
}

\maketitle

\pagestyle{plain}

\vspace{-3mm}
\begin{abstract} The problem of deciding whether CSP instances admit solutions has been deeply
studied in the literature, and several structural tractability results have been derived so far. However,
constraint satisfaction comes in practice as a computation problem where the focus is either on finding one
solution, or on enumerating all solutions, possibly projected to some given set of {output variables}.
The paper investigates the structural tractability of the problem of enumerating (possibly projected) solutions, where tractability means here
computable with polynomial delay (WPD), since in general exponentially many solutions may be computed. A general framework based on the notion
of tree projection of hypergraphs is considered, which generalizes all known decomposition methods.
Tractability results have been obtained both for classes of structures where output variables are part of their specification, and for classes
of structures where computability WPD must be ensured for any possible set of output variables. These results are shown to be tight, by
exhibiting dichotomies for classes of structures having bounded arity and where the tree decomposition method is considered.
\end{abstract}

\section{Introduction}

\subsection{Constraint Satisfaction and Decomposition Methods}

\emph{Constraint satisfaction} is often formalized as a homomorphism problem that takes as input two finite relational structures $\A$
(modeling variables and scopes of the constraints) and $\B$ (modeling the relations associated with constraints), and asks whether there is a
homomorphism from $\A$ to $\B$. Since the general problem is $\NP$-hard, many restrictions have been considered in the literature, where the
given structures have to satisfy additional conditions. In this paper, we are interested in restrictions imposed on the (usually said)
left-hand structure, i.e., $\A$ must be taken from some suitably defined class $\boA$ of structures, while $\B$ is any arbitrary structure from
the class ``$\all$'' of all finite structures.\footnote{Note that the finite property is a feature of this framework, and not a simplifying
assumption. E.g., on structures with possibly infinite domains, the open question in~\cite{GS84} (just recently answered by \cite{GS10} on
finite structures) would have been solved in 1993~\cite{SS93}.} Thus, we face the so-called \emph{uniform} constraint satisfaction problem,
shortly denoted as $\mbox{CSP}(\boA,\all)$, where both structures are part of the input (nothing is fixed).

The decision problem  $\mbox{CSP}(\boA,\all)$ has intensively been studied in the literature, and various classes of structures over which it
can be solved in polynomial time have already been singled out (see \cite{cjg-08,gott-etal-00,grohe-marx-06,adler08}, and the references
therein). These approaches, called {\em decomposition methods}, are based on properties of the hypergraph $\HG_\A$ associated with each
structure $\A\in\boA$. In fact, it is well-known that, for the class $\boA_a$ of all structures whose associated hypergraphs are acyclic,
$\mbox{CSP}(\boA_a,\all)$ is efficiently solvable by just enforcing \emph{generalized arc consistency} ($\GAC$)---roughly, by filtering
constraint relations until every pair of constraints having some variables $\bar X$ in common agree on $\bar X$ (that is, they have precisely
the same set of allowed tuples of values on these variables $\bar X$).

Larger ``islands of tractability'' are then identified by generalizing hypergraph acyclicity. To this end, every decomposition method {\tt DM}
associates with any hypergraph $\HG_\A$ some measure $w$ of its cyclicity, called the {\tt DM}-\emph{width} of $\HG_\A$. The tractable classes
$\boA$ of instances (according to {\tt DM}) are those (with hypergraphs) having bounded width, that is, whose degree of cyclicity is below some
fixed threshold. For every instance $\A$ in such a class $\boA$ and every structure $\B$, the instance $(\A,\B)$ can be solved in
polynomial-time by exploiting the solutions of a set of suitable subproblems, that we call {\em views}, each one solvable in polynomial-time
(in fact, exponential in the---fixed---width, for all known methods). In particular, the idea is to arrange some of these views in a tree,
called decomposition, in order to exploit the known algorithms for acyclic instances. In fact, whenever such a tree exists, instances can be
solved by just enforcing ${\GAC}$ on the available views, even without computing explicitly any decomposition. This very general approach
traces back to the seminal database paper~\cite{GS84}, and it is based on the graph-theoretic notion of {\em tree-projection} of the pair of
hypergraphs $(\HG_\A,\HG_\V)$, associated with the input structure $\A$ and with the structure $\V$ of the available views, respectively (tree
projections are formally defined in Section~\ref{sec:framework}).

For instance, assume that the fixed threshold on the width is $k$: in the {\em generalized hypertree-width} method~\cite{gott-etal-03}, the
available views are all subproblems involving at most $k$ constraints from the given CSP instance; in the case of {\em treewidth}~\cite{RS84},
the views are all subproblems involving at most $k$ variables; for {\em fractional hypertree-width}, the views are all subproblems having
fractional cover-width at most $k$  (in fact, if we require that they are computable in polynomial-time, we may instead use those subproblems
defined in~\cite{M09} to compute a $O(k^3)$ approximation of this notion).

Note that, for the special case of generalized hypertree-width, the fact that enforcing $\GAC$ on all clusters of $k$ constraints is sufficient
to solve the given instance, without computing a decomposition, has been re-derived in~\cite{CD05} (with proof techniques different from those
in~\cite{GS84}). Moreover, \cite{CD05} actually provided a stronger result, as it is proved that this property holds even if there is some
homomorphically equivalent subproblem having generalized hypertree-width at most $k$. However, the corresponding {\em only if} result is
missing in that paper, and characterizing the precise power of this $\GAC$ procedure for the views obtained from all clusters of $k$
constraints (short: $k$-$\GAC$) remained an open question. For any class $\boA$ of instances having bounded arity (i.e., with a fixed maximum
number of variables in any constraint scope of every instance of the class), the question has been answered in~\cite{ABD07}: $\forall
\A\in\boA$, $k$-$\GAC$ is correct for every right-hand structure $\B$ if, and only if, the core of $\A$ has tree width at most $k$ (recall that
treewidth and generalized hypertree-width identify the same set of bounded-arity tractable classes).
In its full version, the answer to this open question follows from a recent result in~\cite{GS10} (see Theorem~2-bis).

In fact, for any recursively enumerable class of bounded-arity structures $\boA$, it is known that this method
is essentially optimal: $\mbox{CSP}(\boA,\all)$ is solvable in polynomial time \emph{if, and only if,} the cores
of the structures in $\boA$ have bounded treewidth (under standard complexity theoretic assumptions)~\cite{G07}.
Note that the latter condition may be equivalently stated as follows: for every $\A\in\boA$ there is some $\A'$
homomorphically equivalent to $\A$ and such that its treewidth is below the required fixed threshold. For short,
we say that such a class has bounded treewidth modulo homomorphic equivalence.

Things with unbounded-arity classes are not that clear. Generalized hypertree-width does not characterize all classes of (arbitrary) structures
where $\mbox{CSP}(\boA,\all)$ is solvable in polynomial time~\cite{grohe-marx-06}. It seems that a useful characterization may be obtained by
relaxing the typical requirement that views are computable in polynomial time, and by requiring instead that such tasks are fixed-parameter
tractable ($\FPT$)~\cite{down-fell-99}. In fact, towards establishing such characterization, it was recently shown in \cite{M10} that (under
some reasonable technical assumptions) the problem  $\mbox{CSP}(\mathcal{H})$, i.e., $\mbox{CSP}(\boA,\all)$ restricted to the instances whose
associated hypergraphs belong to the class $\mathcal{H}$, is $\FPT$ \emph{if, and only if,} hypergraphs in $\mathcal{H}$ have bounded
\emph{submodular} width---a new hypergraph measure more general than \emph{fractional} hypertree-width and, hence, than generalized
hypertree-width.

It is worthwhile noting that the above mentioned tractability results for classes of instances defined modulo homomorphically equivalence are
actually tractability results for the {\em promise} version of the problem. In fact, unless $\Pol=\NP$, there is no polynomial-time algorithm
that may check whether a given instance $\A$ actually belongs to such a class $\boA$. In particular, it has been observed by different
authors~\cite{SGG08,BDGM09} that there are classes of instances having bounded treewidth modulo homomorphically equivalence for which answers
computable in polynomial time cannot be trusted. That is, unless $\Pol=\NP$, there is no efficient way to distinguish whether a ``yes'' answer
means that there is some solution of the problem, or that $\A\not\in\boA$.

In this paper, besides promise problems, we also consider the so-called {\em no-promise} problems, which seem
more appealing for practical applications. In this case, either certified solutions are computed, or the promise
$\A\in\boA$ is correctly disproved. For instance, the algorithm in~\cite{CD05} solves the no-promise
search-problem of computing a homomorphism for a given CSP instance $(\A,\B)$. This algorithm either computes
such a homomorphism or concludes that $\HG_\A$ has generalized hypertree-width greater than $k$.

\subsection{Enumeration Problems}

While the structural tractability of deciding whether CSP instances admit solutions has been deeply studied in
the literature, the structural tractability of the corresponding computation problem received considerably less
attention so far~\cite{BDGM09}, though this is certainly a more appealing problem for practical applications. In
particular, it is well-known that for classes of CSPs where the decision problem is tractable and a
self-reduction argument applies the enumeration problem is tractable too~\cite{DI92,C04}. Roughly, these classes
have a suitable closure property such that one may fix values for the variables without going out of the class,
and thus may solve the computation problem by using the (polynomial-time) algorithm for the decision problem as
an oracle. In fact, for the non-uniform CSP problem, the tractability of the decision problem always entails the
tractability of the search problem~\cite{C04}. As observed above, this is rather different from what happens in
the uniform CSP problem that we study in this paper, where this property does not hold (see~\cite{SGG08,BDGM09},
and Proposition~\ref{prop:NP}), and thus a specific study for the computation problem is meaningful and
necessary.

In this paper, we embark on this study, by focusing on the problem $\ECSP$ of enumerating (possibly projected) solutions. Since even easy
instances may have an exponential number of solutions, tractability means here having algorithms that compute solutions \emph{with polynomial
delay} (WPD): An algorithm $\rm M$ solves WPD a computation problem $\rm P$ if there is a polynomial $p(\cdot)$ such that, for every instance
of $\rm P$ of size $n$,  $\rm M$ discovers if there are no solutions in time $O(p(n))$; otherwise, it outputs all  solutions in such a way that
a new solution is computed within $O(p(n))$ time from the previous one.

Before stating our contribution, it is worthwhile noting that there are different facets of the enumeration problem, and thus different
research directions to be explored:

\smallskip

({\em Which Decomposition Methods?}) We considered the more general framework of the tree projections, where subproblems (views) may be
completely arbitrary, so that our results are smoothly inherited by all (known) decomposition methods. We remark that this choice posed
interesting technical challenges to our analysis, and called for solution approaches that were not explored in the earlier literature on
traditional methods, such as treewidth. For instance, in this general context, we cannot speak anymore of ``the core'' of a structure, because
different isomorphic cores may have different structural properties with respect to the available views.

\smallskip

({\em Only full solutions or possibly projected solutions?}) In this paper, an $\ECSP$ instance is a triple
$(\A,\B,O)$, for which we have to compute all solutions (homomorphisms) projected to a set of desired output
variables $O$, denoted by $\A^\B[O]$. We believe this is the more natural approach. Indeed, modeling real-world
applications through CSP instances typically requires the use of ``auxiliary'' variables, whose precise values
in the solutions are not relevant for the user, and that are (usually) filtered-out from the output. In these
cases, computing all combinations of their values occurring in solutions means wasting time, possibly
exponential time. Of course, this aspect is irrelevant for the problem of computing just one solution, but is
crucial for the enumeration problem.

\smallskip

({\em Should classes of structures be aware of output variables?}) This is an important technical question. We are interested in identifying
classes of tractable instances based on properties of their left-hand structures, while right-hand structures have no restrictions. What about
output variables? In principle, structural properties may or may not consider the possible output variables, and in fact both approaches have
been explored in the literature (see, e.g.,~\cite{G07}), and both approaches are dealt with in this paper. In the former output-aware case,
possible output variables are suitably described in the instance structure. Unlike previous approaches that considered additional ``virtual''
constraints covering together all possible output variables~\cite{G07}, in this paper possible output variables are described as those
variables $X$ having a domain constraint $\dom(X)$, that is, a distinguished unary constraint specifying the domain of this variable. Such
variables are said domain restricted.
In fact, this choice reflects the classical approach in constraint satisfaction systems, where variables are typically associated with domains,
which are heavily exploited by constraint propagation algorithms. Note that this approach does not limit the number of solutions, while in the
tractable classes considered in~\cite{G07} only instances with a polynomial number of (projected) solutions may be dealt with.
As far as the latter case of arbitrary sets of output variables is considered, observe that in general stronger conditions are expected to be
needed for tractability. Intuitively, since we may focus on any desired substructure, no strange situations may occur, and the full instance
should be really tractable.

\subsection{Contribution}

\paragraph{Output-aware classes of $\ECSP$s:}
\begin{description}
  \item[(1)] We define a property for pairs $(\A,O)$, where $\A$ is a structure and $O\subseteq A$ is a set of variables, that allows us to
      characterize the classes of tractable instances. Roughly, we say that $(\A,O)$ is $\tpCovered$ through the decomposition method {\tt
      DM} if variables in $O$ occur in a tree projection of a certain hypergraph w.r.t. to the (hypergraph associated with the) views
      defined according to {\tt DM}.
  \item[(2)] We describe an algorithm that solves the promise enumeration problem, by computing with polynomial delay all solutions of a given instance $(\A,\B,O)$, whenever $(\A,O)$ is $\tpCovered$ through {\tt DM}.
  \item[(3)] For the special case of (generalized hyper)tree width, we show that the above condition is also necessary for the correctness of the proposed algorithm (for every $\B$). In fact, for these traditional decomposition methods we now have a complete characterization of the power of the $k$-$\GAC$ approach.
  \item[(4)] For recursively enumerable classes of structures having bounded arity, we exhibit a dichotomy showing that the above tractability result is tight, for {\tt DM} = treewidth
      (and assuming $\FPT\neq W[1]$).
\end{description}

\paragraph{$\ECSP$ instances over arbitrary output variables:}
\begin{description}
  \item[(1)] We describe an algorithm that, on input $(\A,\B,O)$, solves the no-promise enumeration problem. In particular, either all solutions are computed, or it infers that there exists no tree projection of $\HG_\A$ w.r.t. $\HG_\V$ (the hypergraph associated with the views defined according to {\tt DM}).
      This algorithm generalizes to the tree projection framework the enumeration algorithm of projected solutions recently proposed for the special case of treewidth~\cite{BDGM09}.

  \item[(2)] Finally, we give some evidence that, for bounded arity classes of instances, we cannot do better than this. In particular, having bounded width tree-decompositions of the full structure seems a necessary condition for enumerating WPD.
We speak of ``evidence,'' instead of saying that our result completely answers the open question
in~\cite{G07,BDGM09}, because our dichotomy theorem focuses on classes of structures satisfying the technical property of being closed under taking minors (in fact, the same property assumed in the first dichotomy result on the complexity of the decision problem on classes of graphs~\cite{GSS01}).
\end{description}

\section{Preliminaries: Relational Structures and Homomorphisms}\label{sec:framework}

A constraint satisfaction problem may be formalized as a relational homomorphism problem. A vocabulary $\tau$ is a finite set of relation
symbols of specified arities. A relational structure $\A$ over $\tau$ consists of a universe $A$ and an $r$-ary relation $R^\A\subseteq A^r$,
for each relation symbol $R$ in $\tau$.

If $\A$ and $\A'$ are two relational structures over disjoint vocabularies, we denote by $\A\uplus\A'$ the relational structure over the
(disjoint) union of their vocabularies, whose domain (resp., set of relations) is the union of those of $\A$ and $\A'$.

A {\em homomorphism} from a relational structure $\A$ to a relational structure $\B$ is a mapping $h: A \mapsto B$ such that, for every
relation symbol $R\in\A$, and for every tuple $\tuple{a_1,\ldots,a_r}\in R^\A$, it holds that $\tuple{h(a_1),\ldots,h(a_r)}\in R^\B$. For any
set $X\subseteq A$, denote by $h[X]$ the restriction of $h$ to $X$. The set of all possible homomorphisms from $\A$ to $\B$ is denoted by
$\A^\B$, while $\A^\B[X]$ denotes the set of their restrictions to $X$.

An instance of the constraint satisfaction problem ($\CSP$) is a pair $(\A,\B)$ where $\A$ is called a \emph{left-hand structure} (short:
$\ell$-structure) and $\B$ is called a \emph{right-hand structure} (short: $r$-structure). In the classical decision problem, we have to decide
whether there is a homomorphism from $\A$ to $\B$, i.e., whether $\A^\B\neq\emptyset$. In an instance of the corresponding enumeration problem
(denoted by $\ECSP$) we are additionally given a set of {\em output} elements $O\subseteq A$; thus, an instance has the form $(\A,\B,O)$. The
goal is to compute the restrictions to $O$ of all possible homomorphisms from $\A$ to $\B$, i.e., $\A^\B[O]$.
If $O=\emptyset$, the computation problem degenerates to the decision one. Formally, let $h_\phi : \emptyset\mapsto {\it true}$ denote (the
constant mapping to) the Boolean value {\it true}; then, define $\A^\B[\emptyset]=\{h_\phi\}$ (resp., $\A^\B[\emptyset]=\emptyset$) if there is
some (resp., there is no) homomorphism from $\A$ to $\B$.

In the constraint satisfaction jargon, the elements of $A$ (the domain of the $\ell$-structure $\A$) are the variables, and there is a
constraint $C=(\tuple{a_1\ldots,a_r},R^\B)$  for every tuple $\tuple{a_1\ldots,a_r}\in R^A$  and every relation symbol $R\in\tau$. The tuple of
variables is usually called the scope of $C$, while $R^\B$ is called the relation of $C$. Any homomorphism from $\A$ to $\B$ is thus a mapping
from the variables in $A$ to the elements in $B$ (often called domain values) that satisfies all constraints, and it is also called a solution
(or a projected solution, if it is restricted to a subset of the variables).

Two relational structures $\A$ and $A'$ are homomorphically equivalent if there is a homomorphism from $\A$ to
$\A'$ and vice-versa. A structure $\A'$ is a substructure of $\A$ if $A'\subseteq A$ and $R^{\A'}\subseteq R^A$,
for each symbol $R\in\tau$. Moreover, $\A'$ is a \emph{core} of $\A$ if it is a substructure of $\A$ such that:
\emph{(1)} there is a homomorphism from $\A$ to $\A'$, and \emph{(2)} there is no substructure $\A''$ of $\A'$,
with $\A''\neq\A'$, satisfying \emph{(1)}.

\section{Decomposition Methods, Views, and Tree Projections}\label{sec:views}

Throughout the following sections we assume that $(\A,\B)$ is a given connected CSP instance, and we shall we
shall seek to compute its solutions (possibly restricted to a desired set of output variables) by combining the
solutions of suitable sets of subproblems, available as additional distinguished constraints called
\emph{views}.

Let $\A_\V$ be an $\ell$-structure with the same domain as $\A$. We say that $\A_\V$ is a \emph{view structure} (short: $v$-structure) if

\vspace{-3mm}\begin{itemize}
\item its vocabulary $\tau_\V$ is disjoint from the vocabulary $\tau$ of $\A$;
\item every relation $R^{\A_\V}$ contains a single tuple whose variables will be denoted by $\vars(R^{\A_\V})$. That is, there is a one-to-one correspondence between views and relation symbols in $\tau_\V$, so that we shall often use the two terms interchangeably;
\item for every relation $R\in \tau$ and every tuple $t\in R^\A$, there is some relation $R_t\in\tau_\V$, called {\em base view}, such that
    $\{t\}=R_t^{\A_\V}$, i.e., for every constraint in $\A$ there is a corresponding view in $\A_\V$.
\end{itemize}

\vspace{-2mm} \noindent Let $\B_\V$ be an $r$-structure. We say that $\B_\V$ is {\em legal} (w.r.t. $\A_\V$ and
$(\A,\B)$) if

\vspace{-3mm}\begin{itemize}
\item its vocabulary is $\tau_\V$;
\item For every view $R\in \tau_\V$, $R^{\B_\V}\supseteq \A^\B[w]$ holds, where $w=\vars(R^{\A_\V})$. That is, every subproblem is not more restrictive than the full problem.
\item For every base view $R_t\in \tau_\V$, $R_t^{\B_\V}\subseteq R^\B$. That is, any base view is at least as restrictive as the ``original'' constraint associated with it.
\end{itemize}


The following fact immediately follows from the above properties.
\begin{fact}
Let $\B_\V$ be any $r$-structure that is legal w.r.t. $\A_\V$ and $(\A,\B)$. Then, $\forall O\subseteq A$, the ECSP instance $(\A_\V,\B_\V,O)$
has the same set of solutions as $(\A,\B,O)$.
\end{fact}

In fact, all structural decomposition methods define some way to build the views to be exploited for solving the given CSP instance. In our
framework, we associate with any decomposition method {\tt DM} a pair of polynomial-time computable functions $\lDM$ and $\rDM$ that, given any
CSP instance $(\A,\B)$, compute the pair $(\A_\V,\B_\V)$, where $\A_\V=\lDM(\A)$ is a $v$-structure, and $\B_\V=\rDM(\A,\B)$ is a legal
$r$-structure.\footnote{A natural extension of this notion we may be to consider $\FPT$ decomposition methods, where functions $\lDM$ and
$\rDM$ are computable in fixed-parameter polynomial-time. For the sake of presentation and of space, we do not consider $\FPT$ decomposition
methods in this paper, but our results can be extended to them rather easily.}

For instance, for any fixed natural number $k$, the \emph{generalized hypertree decomposition method}~\cite{gott-etal-99} (short: $hw_k$) is
associated with the functions $\it \ell\mbox{-}hw_k$ and $\it r\mbox{-}hw_k$ that, given a CSP instance $(\A,\B)$, build the pair $({\it
\ell\mbox{-}hw_k(\A)},{\it r\mbox{-}hw_k(\A,\B)})$ where, for each subset $\mathcal{C}$ of at most $k$ constraints from $(\A,\B)$, there is a
view $R_{\mathcal{C}}$ such that: (1) $\vars(R_{\mathcal{C}}^{\it \ell\mbox{-}hw_k(\A)})$ is the set of all variables occurring in
$\mathcal{C}$, and (2) the tuples in $R_{\mathcal{C}}^{\it r\mbox{-}hw_k(\A,\B)}$ are the solutions of the subproblem encoded by $\mathcal{C}$.
Similarly, the \emph{tree decomposition method}~\cite{RS84} ($tw_k$) is defined as above, but we consider each subset of at most $k$ variables
in $\A$ instead of each subset of at most $k$ constraints.

\begin{figure*}[t]
  \centering
  \includegraphics[width=\textwidth]{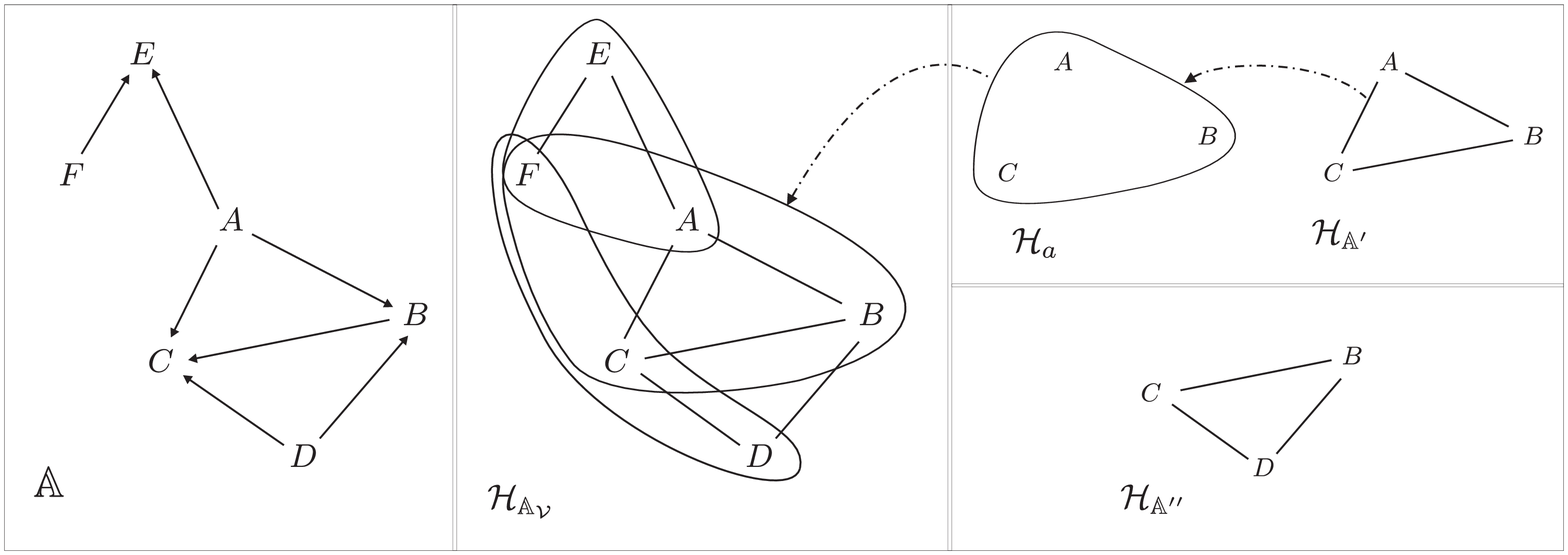}\vspace{-3mm}
  \caption{A structure $\A$. A hypergraph $\HG_{\A_\V}$ such that $(\HG_{\A},\HG_{\A_\V})$ has no tree
  projections. Two hypergraphs $\HG_{\A'}$ and $\HG_{\A''}$, where $\A'$ and $\A''$ are cores of
  $\A$. A tree projection $\HG_a$ of $(\HG_{\A'},\HG_{\A_V})$.
  }\label{fig:triangle}\vspace{-2mm}
\end{figure*}

\subsection{Tree Projections for CSP Instances}

In this paper we are interested in restrictions imposed on left-hand structures of CSP instances, based on some decomposition method {\tt DM}.
To this end, we associate with any $\ell$-structure $\A$ a hypergraph $\HG_\A=(A,H)$, whose set of nodes is equal to the set of variables $A$
and where, for each constraint scope in $R^\A$, the set $H$ of hyperedges contains a hyperedge including all its variables (no further
hyperedge is in $H$).
In particular, the $v$-structure $\A_\V=\lDM(\A)$ is associated with a hypergraph $\HG_{\A_\V}=(A,H)$, whose set of nodes is the set of
variables $A$ and where, for each view $R\in \tau_\V$, the set $H$ contains the hyperedge $\vars(R^{\A_\V})$.
In the following, for any hypergraph $\HG$, we denote its nodes and edges by $\nodes(\HG)$ and $\edges(\HG)$, respectively.

\begin{example}\label{ex:main}
Consider the $\ell$-structure $\A$ whose vocabulary just contains the binary relation symbol $\it R$, and such that
$\mathit{R}^\A=\{\tuple{F,E},$ $\tuple{A,E},$ $\tuple{A,C},$ $\tuple{A,B},$ $\tuple{B,C},$ $\tuple{D,B},$ $\tuple{D,C}\}$. Such a simple
one-binary-relation structure may be easily represented by the directed graph in the left part of Figure~\ref{fig:triangle}, where edge
orientation reflects the position of the variables in $R$. In this example, the associated hypergraph $\HG_{\A}$ is just the undirected version
of this graph.
Let {\tt DM} be a method that, on input $\A$, builds the $v$-structure $\A_\V=\lDM(\A)$ consisting of the seven
base views of the form $R_t$, for each tuple $t\in R^\A$, plus the three relations $R_1$, $R_2$, and $R_3$ such
that $R_1^{\A_\V}=\{\tuple{A,E,F}\}$, $R_2^{\A_\V}=\{\tuple{A,B,C,F}\}$, and $R_3^{\A_\V}=\{\tuple{C,D,F}\}$.
Figure~\ref{fig:triangle} also reports $\HG_{\A_\V}$. \hfill $\lhd$
\end{example}

A hypergraph $\HG$ is {\em acyclic} iff it has a {\em join tree}~\cite{bern-good-81}, i.e., a tree $\JT(\HG)$, whose vertices are the
hyperedges of $\HG$, such that if a node $X$ occurs in two hyperedges $h_1$ and $h_2$ of $\HG$, then $h_1$ and $h_2$ are connected in
$\JT(\HG)$, and $X$ occurs in each vertex on the unique path linking $h_1$ and $h_2$ in $\JT(\HG)$.

For two hypergraphs $\HG_1$ and $\HG_2$, we write $\HG_1\leq \HG_2$ iff each hyperedge of $\HG_1$ is contained in at least one hyperedge of
$\HG_2$. Let $\HG_1\leq \HG_2$. Then, a \emph{tree projection} of $\HG_1$ with respect to $\HG_2$ is an acyclic hypergraph $\HG_a$ such that
$\HG_1\leq \HG_a \leq \HG_2$. Whenever such a hypergraph $\HG_a$ exists, we say that the pair $(\HG_1,\HG_2)$ has a tree projection (also, we
say that $\HG_1$ has a tree projection w.r.t. $\HG_2$). The problem of deciding whether a pair of hypergraphs has a tree projection is called
the \emph{tree projection problem}, and it has recently been proven to be $\NP$-complete~\cite{GMS07}.

\begin{example}\label{ex:main2} Consider again the setting of Example~\ref{ex:main}.
It is immediate to check that the pair of hypergraphs $(\HG_\A,\HG_{\A_\V})$ does not have any tree projection. Consider instead the
(hyper)graph $\HG_{\A'}$ reported on the right of Figure~\ref{fig:triangle}. The acyclic hypergraph $\HG_a$ is a tree projection of $\HG_{\A'}$
w.r.t. $\HG_{\A_\V}$. In particular, note that the hyperedge $\{A,B,C\}\in\edges(\HG_a)$ ``absorbs'' the cycle in $\HG_{\A'}$, and that
$\{A,B,C\}$ is in its turn contained in the hyperedge $\{A,B,C,F\}\in\edges(\HG_{\A_\V})$. \hfill $\lhd$
\end{example}

Note that all the (known) structural decomposition methods can be recast as special cases of tree projections, since they just differ in how
they define the set of views to be built for evaluating the CSP instance. For instance, a hypergraph $\HG_\A$  has generalized hypertree width
(resp., treewidth) at most $k$ if and only if there is a tree projection of $\HG_\A$ w.r.t. $\HG_{\it \ell\mbox{-}hw_k(\A)}$ (resp., w.r.t.
$\HG_{\it \ell\mbox{-}tw_k(\A)}$).

However, the setting of tree projections is more general than such traditional decomposition approaches, as it allows us to consider arbitrary
sets of views, which often require more care and different techniques. As an example, we shall illustrate below that in the setting of tree
projections it does not make sense to talk about ``the'' core of an $\ell$-structure, because different isomorphic cores may differently behave
with respect to the available views. This phenomenon does not occur, e.g., for generalized hypertree decompositions, where all combinations of
$k$ constraints are available as views.

\begin{example}\label{ex:main3}
Consider the structure $\A$ illustrated in Example~\ref{ex:main}, and the structures $\A'$ and $\A''$ over the same vocabulary as $\A$, and
such that $R^{\A'}=\{\tuple{A,C},$ $\tuple{A,B},$ $\tuple{B,C}\}$ and $R^{\A''}=\{\tuple{B,C},$ $\tuple{D,B},$ $\tuple{D,C}\}$. The hypergraphs
$\HG_{\A'}$ and $\HG_{\A''}$ are reported in Figure~\ref{fig:triangle}.
Note that $\A'$ and $\A''$ are two (isomorphic) cores of $\A$, but they have completely different structural properties. Indeed,
$(\HG_{\A'},\HG_{\A_\V})$ admits a tree projection (recall Example~\ref{ex:main2}), while $(\HG_{\A''},\HG_{\A_\V})$ does not. \hfill $\lhd$
\end{example}

\subsection{CSP Instances and tp-Coverings}

We complete the picture of our unifying framework to deal with decomposition methods for constraint satisfaction problems, by illustrating some
recent results in \cite{GS10}, which will be useful to our ends. Let us start by stating some preliminary definitions.

For a set of variables $O=\{X_1,\ldots,X_r\}$, let $\mathbb{S}_{O}$ denote the structure with a fresh $r$-ary relation symbol $R_{O}$ and
domain $O$, such that $R_{O}^{\A_{O}}=\{\tuple{X_1,\ldots,X_r}\}$.

\begin{definition}\label{def:tpcovered}\em
Let $\A_\V$ be a $v$-structure. A set of variables $O\subseteq A$ is $\tpCovered$ in $\A_\V$ if there exists a core $\A'$ of
$\A\uplus\mathbb{S}_{O}$ such that $(\HG_{\A'},\HG_{\A_\V})$ has a tree projection.\footnote{For the sake of completeness, note that we use
here a core $\A'$ because we found it more convenient for the presentation and the proofs. However, it is straightforward to check that this
notion can be equivalently stated in terms of any structure homomorphically equivalent to $\A\uplus\mathbb{S}_{O}$. The same holds for the
related Definition~\ref{def:tpcovered-new}.}\hfill $\Box$
\end{definition}

For instance, it is easily seen that the variables $\{A,B,C\}$ are $\tpCovered$ in the $v$-structure $\A_\V$ discussed in
Example~\ref{ex:main}. In particular, note that the structure $\A\uplus\mathbb{S}_{\{A,B,C\}}$ is associated with the same hypergraph $\HG_{\A'}$ that has a
tree projection w.r.t. $\HG_{\A_\V}$ (cf. Example~\ref{ex:main3}). Instead,  the variables $\{B,C,D\}$ are not $\tpCovered$ in $\A_\V$.

Given a CSP instance $(\A_\V,\B_\V)$, we denote by  $\GAC(\A_\V,\B_\V)$ the $r$-structure that is obtained by enforcing generalized arc
consistency on $(\A_\V,\B_\V)$.

The following result, proved in~\cite{GS10} for a different setting, states the precise relationship between generalized-arc-consistent views
and $\tpCovered$ sets of variables.

\begin{theorem}\label{thm:pods}
Let $\A$ be an $\ell$-structure, and let $\A_\V$ be a $v$-structure. The following are equivalent:

\vspace{-1mm}
\begin{enumerate}
\item[(1)] A set of variables $O\subseteq A$ is $\tpCovered$ in $\A_\V$;

\item[(2)] For every $r$-structure $\B$, for every $r$-structure $\B_\V$ that is legal w.r.t. $\A_\V$ and $(\A,\B)$, and for every relation $R\in\tau_\V$ with $O\subseteq \vars(R^{\A_\V})$, $R^{\GAC(\A_\V,\B_\V)}[O]=\A^\B[O]$.
\end{enumerate}
\end{theorem}

Note that the result answered a long standing open question \cite{GS84,SS93} about the relationship between the existence of tree projections
and (local and global) consistency properties in databases~\cite{GS10}.
In words, the result states that just enforcing generalized arc consistency on the available views is a sound and complete procedure to solve
ECSP instances \emph{if, and only if,} we are interested in (projected) solutions over output variables that are $\tpCovered$ and occur together
in some available view.
Thus, in these cases, all solutions can be computed in polynomial time.
The more general case where output variables are arbitrary (i.e., not necessarily included in some available view) is explored in the rest of
this paper.

We now leave the section by noticing that as a consequence of Theorem~\ref{thm:pods}, we can characterize the
power of local-consistency for any decomposition method {\tt DM} such that, for each pair $(\A,\B)$, each view
in $\B_\V=\rDM(\A,\B)$ contains the solutions of the subproblem encoded by the constraints over which it is
defined. For the sake of simplicity, we state below the result specialized to the well-known decomposition
methods $tw_k$ and $hw_k$.
\smallskip

\noindent \textbf{Theorem 2-bis.} {\em Let {\tt DM} be a decomposition method in $\{tw_k,hw_k\}$, let $\A$ be an $\ell$-structure, and let
$\A_\V=\lDM(\A)$. The following are equivalent:

\vspace{-1mm}
\begin{enumerate}
\item[(1)] A set of variables $O\subseteq A$ is $\tpCovered$ in $\A_\V$;

\item[(2)] For every $r$-structure $\B$, and for every relation $R\in\tau_\V$ with $O\subseteq \vars(R^{\A_\V})$,
    $R^{\GAC(\A_\V,\B_\V)}[O]=\A^\B[O]$,  where $\B_\V=\rDM(\A,\B)$.
\end{enumerate}
}
\begin{proof}[Sketch]
Preliminarily, it is easy to see that (2) in Theorem~\ref{thm:pods} may be equivalently stated as follows:
\begin{enumerate}
\item[(2')] For every $r$-structure $\B$, for every $r$-structure $\B_\V$ that is legal w.r.t. $\A_\V$ and $(\A,\B)$ and such that
    $\B_\V=\GAC(\A_\V,\B_\V)$, and for every relation $R\in\tau_\V$ with $O\subseteq \vars(R^{\A_\V})$, $R^{\B_\V}[O]=\A^\B[O]$.
\end{enumerate}

The fact that $ (1) \Rightarrow (2)$ trivially follows from Theorem~\ref{thm:pods}. We have to show that $(2) \Rightarrow (1)$ holds as well.
To this end, observe that if $O$ is not $\tpCovered$ in $\lDM(\A)$, by Theorem~\ref{thm:pods} (actually, $(1) \Rightarrow(2')$), we can
conclude the existence of: (1) an $r$-structure $\B$, (2) an $r$-structure $\B_\V$ that is legal w.r.t. $\A_\V$ and $(\A,\B)$ and such that
$\B_\V=\GAC(\A_\V,\B_\V)$, and (3) a relation $R\in\tau_\V$ with $O\subseteq \vars(R^{\A_\V})$ such that $R^{\B_\V}[O]\neq \A^\B[O]$ (of
course, $R^{\B_\V}[O]\supset \A^\B[O]$ by the legality of $\B_\V$).
Consider now the $r$-structure $\B_\V'={\it r\mbox{-}hw_k(\A,\B)}$. Recall that each view in $\B_\V'$ contains \emph{all} the solutions of the
subproblem encoded by the constraints over which it is defined. Since $\B_\V=\GAC(\A_\V,\B_\V)$, it can be shown that each view in $\B_\V$
contains \emph{only} solutions of the subproblem encoded by the constraints over which it is defined. Thus, for each relation $R\in\tau_\V$,
$R^{\B_\V'}\supseteq R^{\B_\V}$ holds, which implies $R^{\B_\V'}[O]\supset \A^\B[O]$. The same line of reasoning applies to the tree
decomposition method.\hfill $\lhd$
\end{proof}

Note that if we consider decision problem instances ($O=\emptyset$) and the treewidth method ($\A_\V={\it \ell\mbox{-}tw_k(\A)}$), from Theorem
2-bis, we (re-)obtain the nice characterization of~\cite{ABD07} about the relationship between $k$-local consistency and treewidth modulo
homomorphic equivalence. If we consider generalized hypertree-width ($\A_\V={\it \ell\mbox{-}hw_k(\A)}$), we get the answer to the
corresponding open question for the unbounded arity case, that is, the precise power of the procedure enforcing $k$-union (of constraints)
consistency (i.e., the power of the algorithm for the decision problem described in~\cite{CD05}).

\section{Enumerating Solutions of Output-Aware CSP Instances}

The goal of this section is to study the problem of enumerating CSP solutions for classes of instances where possible output variables are part
of the structure of the given instance. This is formalized by assuming that the relational structure contains {\em domain constraints}  that
specify the domains for such variables.

\begin{definition}\em
A variable $X\in A$ is {\em domain restricted} in the $\ell$-structure $\A$ if there exists a unary distinguished (domain) relation symbol
$\dom{(X)}\in \tau$ such that $\{\tuple{X}\}=\dom{(X)}^\A$. The set of all domain restricted variables is denoted by $\DRV(\A)$. \hfill $\Box$
\end{definition}

We say that an ECSP instance $(\A,\B,O)$ is domain restricted if $O\subseteq\DRV(\A)$. Of course, if it is not, then one may easily build in
linear time an equivalent domain-restricted ECSP instance where an additional fresh unary constraint is added for every output variable, whose
values are taken from any constraint relation where that variable occurs. We say that such an instance is a domain-restricted version of
$(\A,\B,O)$.

Figure~\ref{fig:algoritmo} shows an algorithm, named \alg, that computes the solutions of a given ECSP instance. The algorithm is parametric
w.r.t. any chosen decomposition method {\tt DM}, and works as follows.
Firstly, \alg\ starts by transforming the instance $(\A,\B,O)$ into a domain restricted one, and by constructing the views in $(\A_\V,\B_\V)$
via {\tt DM}. Then, it invokes the procedure ${\tt Propagate}$.
This procedure backtracks over the output variables $\{X_1,\ldots, X_m\}$: At each step $i$, it tries to assign
a value to $X_i$ from its domain view,\footnote{With an abuse of notation, in the algorithm we denote by
$\dom(X)$ the base view in $\tau_\V$ associated with the input constraint $\dom(X)\in\tau$ (in fact, no
confusion may arise because the algorithm only works on views).} and defines this value as the unique one
available in that domain, in order to ``propagate'' such an assignment over all other views. This is
accomplished by enforcing generalized arc-consistency each time the procedure is invoked. Eventually, whenever
an assignment is computed for all the variables in $O$, this solution is returned in output, and the algorithm
proceeds by backtracking again trying different values.

\begin{figure}[t]
\centering \fbox{
\parbox{0.68\textwidth}{
\begin{tabular}{l}
  \textbf{Input}: An ECSP instance $(\A,\B,O)$, where $O=\{X_1,\ldots,X_m\}$;\\
  \textbf{Output}: $\A^\B[O]$;\\
  \textbf{Method}: update $(\A,\B,O)$ with any of its domain-restricted versions;\\
  \ \ \ \ \ \ \ \ \ \ \ \ \ \ \ \ let $\A_\V:=\lDM(\A)$, \ \ $\B_\V:=\rDM(\A,\B)$;\\
  \ \ \ \ \ \ \ \ \ \ \ \ \ \ \ \ invoke ${\tt Propagate}$$(1,(\A_\V,\B_\V),m,\tuple{})$; \\
  \hline\\
  \vspace{-7.4mm}\\
  \hline\\
  \vspace{-7mm}\\
  \textbf{Procedure} ${\tt Propagate}$($i$: integer, $(\A_\V,\B_\V)$: pair of structures, $m$: integer,\\
  \hspace{29mm}      \ $\tuple{a_1,...,a_{i-1}}$: tuple of values in $A^i$);    \\
  \textbf{begin}\\
  \textsc{1.}\ \ \ let $\B'_\V:=\GAC(\A_\V,\B_\V)$;  \\
  \textsc{2.}\ \ \ let ${\it activeValues}:=\dom{(X_i)}^{\B'_\V}$;  \\
  \textsc{3.}\ \ \ \textbf{for each} element $\tuple{a_i}\in {\it activeValues}$ \textbf{do}\\
  \textsc{4.}\ \ \ $\mid$ \ \ \ \ \textbf{if} $i=m$ \textbf{then}\\
  \textsc{5.}\ \ \ $\mid$ \ \ \ \ $\mid$ \ \ \ \textbf{output} $\tuple{a_1,...,a_{m-1},a_m}$; \\
  \textsc{6.}\ \ \ $\mid$ \ \ \ \ \textbf{else}\\
  \textsc{7.}\ \ \ $\mid$ \ \ \ \ $\mid$ \ \ \ update $\dom{(X_i)}^{\B'_\V}$ with $\{\tuple{a_i}\}$;
            \quad \quad {\em /$\ast$ $X_i$ is fixed to value $a_i$ }$\ast$/\\
  \textsc{8.}\ \ \ \hspace{-0.2mm}$\lfloor$\hspace{0.4mm} \ \ \ $\lfloor$ \ \ \ ${\tt Propagate}$$(i+1,(\A_\V,\B'_\V),m,\tuple{a_1,...,a_{i-1},a_i})$; \\
  \textbf{end.}\\
\end{tabular}
}}\vspace{-2mm}
  \caption{\textbf{Algorithm} \alg.} \label{fig:algoritmo}\vspace{-3mm}
\end{figure}

\subsection{Tight Characterizations for the Correctness of \alg}

To characterize the correctness of \alg, we need to define a structural property that is related to the one stated in
Definition~\ref{def:tpcovered}. Below, differently from Definition~\ref{def:tpcovered} where the set of output variables $O$ is treated as a
whole, each variable in $O$ has to be tp-covered as a singleton set.

\begin{definition}\label{def:tpcovered-new}\em
Let $(\A,\B,O)$ be an ECSP instance. We say that $(\A,O)$ is $\tpCovered$ through {\tt DM} if there is a core $\A'$ of $\A\uplus\biguplus_{X\in
O} \mathbb{S}_{\{X\}}$ such that $(\HG_{\A'},\HG_{\small \lDM(\A)})$ has a tree projection.~\hfill~$\Box$
\end{definition}

Note that the above definition is purely structural, because (the right-hand structure) $\B$ plays no role there.
In fact, we next show that this definition captures classes of instances where \alg\ is correct.

\begin{theorem}\label{thm:correct}
Let {\tt DM} be a decomposition method, let $\A$ be an $\ell$-structure, and let $O\subseteq A$ be a set of variables. Assume that $(\A,O)$ is
$\tpCovered$ through {\tt DM}. Then, for every $r$-structure $\B$, \alg\ computes the set $\A^\B[O]$.
\end{theorem}

\begin{proof}[Sketch]
Let $\B_{in}$ be any $r$-structure. Preliminarily observe that if the original input instance, say $I_{in}=(\A_{in},\B_{in},O)$, is
$\tpCovered$ through {\tt DM}, the same property is enjoyed by its equivalent domain-restricted version, say $I_0=(\A,\B_0,O)$, computed in the
starting phase of the algorithm. Thus, there is a core $\A'$ of $\A\uplus\biguplus_{X\in O} \mathbb{S}_{\{X\}}$ such that
$(\HG_{\A'},\HG_{\A_\V})$ has a tree projection, where $\A_\V=\lDM(\A)$. This entails that, $\forall X\in O$, $\{X\}$ is $\tpCovered$ in
$\HG_{\A_\V}$.
It is sufficient to show that, if $\A^\B_0\neq\emptyset$, at the generic call of ${\tt Propagate}$ with $i$ as its first argument, ${\it active
Values}$ is initialized at Step~2 with a non-empty set that contains all those values that $X_i\in O$ may take, in any solution of $(\A,\B_0)$
extending the current partial solution $\tuple{a_1,...,a_{i-1}}$; otherwise ($\A^{\B_0}=\emptyset$), ${\it active Values}=\emptyset$, and the
algorithm correctly terminates with an empty output without ever entering the {\bf for} cycle.
For the sake of presentation, we just prove what happens in the first call.
The generalization to the generic case is then straightforward.

Let $i=1$ and assume that $\B'_\V:=\GAC(\A_\V,\B_\V)$ has been computed. From the $\tpCovered$ property of
variables in $O$ and Theorem~\ref{thm:pods}, it follows that $\forall X\in O$, its domain view $\dom(X)$ is such
that $\dom(X)^{\B'_\V}=\A^{\B_0}[\{X\}]$. Thus, all values in the domain views associated with output variables
occur in some solutions. This holds in particular for $\dom(X_1)$ that is empty if, and only if,
$\A^{\B_0}=\emptyset$, in which case the cycle is skipped and the algorithm immediately halts with an empty
output. Assume now that this is not the case, so that ${\it active Values}= \A^{\B_0}[\{X_1\}]\neq\emptyset$,
and let $a_1$ be the chosen value at Step~3. Consider a new instance $I_1=(\A,\B_1,O)$ where the domain
constraint for $X_1$ contains the one value $a_1$. From the above discussion it follows that
$\A^{\B_1}\neq\emptyset$, and clearly the solutions of $I_1$ are all and only those of $I_0$ that extend the
partial solution $\tuple{a_1}$. Moreover, it is easy to check that the $r$-structure ${\B'_\V}$ obtained after
the execution of Step~7 is legal w.r.t. $(\A,\B_1)$, and recall that nothing is changed in the pair $(\A,O)$,
which is (still) $\tpCovered$ through {\tt DM}. Therefore, when we call recursively call ${\tt Propagate}$ at
Step~8 with $i=2$, we are in the same situation as in the first call, but going to enumerate the solutions of
$I_1$. At the end of this call, we just repeat this procedure with the next available value for $X_1$, say
$a_2$, until all elements in ${\it active Values}=\A^\B_0[\{X_1\}]$ have been considered (and
propagated).~\hfill~$\lhd$
\end{proof}

We now complete the picture by observing that Definition~\ref{def:tpcovered-new} also provides the necessary
conditions for the correctness of \alg. As in Theorem~2-bis, we state below the result specialized to the
methods $tw_k$ and $hw_k$.

\begin{theorem}\label{thm:correct2}
Let {\tt DM} be a decomposition method in $\{tw_k,hw_k\}$, let $\A$ be an $\ell$-structure, and let $O\subseteq A$ be a set of variables.
Assume that, for every $r$-structure $\B$, \alg\ computes $\A^\B[O]$. Then, $(\A,O)$ is $\tpCovered$ through {\tt DM}.
\end{theorem}
\begin{proof}[Sketch]
Assume that $(\A,O)$ is not $\tpCovered$ through {\tt DM}, and let $O'\subseteq O$ be a maximal set of output variables such that $(\A,O')$ is
$\tpCovered$ through {\tt DM}. In the case where $O'=\emptyset$, there is no core $\A'$ of $\A$ such that $(\HG_{\A'},\HG_{\lDM(\A)})$ has a
tree projection. Thus, we can apply Theorem~2-bis and conclude that there are an $r$-structure $\B$, and a relation $R\in\tau_\V$ such that
$\A^\B$ has a solution while $R^{\GAC(\lDM(\A),\B_\V})$ is empty, with $\B_\V=\rDM(\A,\B)$. It follows that \alg\ will not produce any output.
Consider now the case where $O'\neq \emptyset$, and where any $X\in O\setminus O'$ is a variable such that $(\A,O'\cup\{X\})$ is not
$\tpCovered$ through {\tt DM}. Let $\bar \A$ be the relational structure $\A\uplus\biguplus_{Y\in O'} \mathbb{S}_{\{Y\}}$, which is such that
$(\HG_{\bar \A},\HG_{\small \lDM(\A)})$ has a tree projection. Then, $\{X\}$ is not tp-covered in $\lDM(\bar \A)$. By Theorem~2-bis,  there are
an $r$-structure $\B$, and a relation $R\in\tau_\V$ with $\{X\}\subseteq \vars(R^{\bar \A_\V})$ such that $R^{\GAC(\bar
\A_\V,\B_\V)}[\{X\}]\supset\bar \A^\B[\{X\}]$, where $\B_\V=\rDM(\bar \A,\B)$. In fact, we can show that such a ``counterexample'' structure
$\B$ can be chosen in such a way that there are (full) solutions $h$ for the problem having the following property: some values in
$R^{\GAC(\bar \A_\V,\B_\V)}[\{X\}]\setminus\bar \A^\B[\{X\}]$ belongs to the generalized arc consistent structure $\B'$ where variables in $O'$
are fixed according to $h[O']$. Thus while enumerating such a solution $h[O']$, \alg\ generates wrong extensions of this solution to the
variable $X$. \hfill $\lhd$
\end{proof}

\subsection{Tight Characterizations for Enumerating Solutions with Polynomial Delay}

We next analyze the complexity of \alg.

\begin{theorem}\label{thm:WPD}
Let $\A$ be an $\ell$-structure, and $O\subseteq A$ be a set of variables. If $(\A,O)$ is $\tpCovered$ through {\tt DM}, then \alg\ runs WPD.
\end{theorem}
\begin{proof}
Assume that $(\A,O)$ is $\tpCovered$ through {\tt DM}. By Theorem~\ref{thm:correct}, we know that \alg\ computes the set $\A^\B[O]$. Thus, if
the algorithm does not output any tuple, we can immediately conclude that the ECSP instance does not have solutions.
Concerning the running time, we preliminary notice that the initialization steps are feasible in polynomial time. In particular, computing
$\A_\V$ and $\B_\V$ is feasible in polynomial time, by the properties of the decomposition method {\tt DM} (see Section~\ref{sec:views}).
To characterize the complexity of the recursive invocations of ${\tt Propagate}$, we have to consider instead two cases.

In the case where there is no solution, we claim that the $r$-structure $\B_\V'$ obtained by enforcing generalized arc consistency in the first
invocation of ${\tt Propagate}$ (i.e., for $i=1$) is empty. Indeed, since $(\A,O)$ is $\tpCovered$ through {\tt DM}, then $\{X_1\}$ is
$\tpCovered$ in $\A_\V$---just compare Definition~\ref{def:tpcovered} and Definition~\ref{def:tpcovered-new}. It follows that we can apply
Theorem~\ref{thm:pods} on the set $\{X_1\}$ in order to conclude that, for every relation $R\in\tau_\V$ with $X_1\in \vars(R^{\A_\V})$,
$R^{\GAC(\A_\V,\B_\V)}[\{X_1\}]=\A^\B[\{X_1\}]$. Since, $\A^\B[O]$ is empty, the above implies that $\GAC(\A_\V,\B_\V)$ is empty too. Thus,
\alg\ invokes just once ${\tt Propagate}$, where the only operation carried out is to enforce generalized arc consistency, which is feasible in
polynomial time.

Consider now the case where $\A^\B[O]$ is not empty. Then, the first solution is computed after $m$ recursive calls of the procedure ${\tt
Propagate}$, where the dominant operation is to enforce generalized arc consistency on the current pair $(\A_\V,\B_\V)$. In particular, by the
arguments in the proof of Theorem~\ref{thm:correct}, it follows that ${\tt Propagate}$ does not have to backtrack to find this solution: after
enforcing generalized arc consistency at step $i$, any active value for $X_i$ is guaranteed to occur in a solution with the current fixed
values for the previous variables $X_j$, $1\leq j<i$. Since $\GAC$ can be enforced in polynomial time, this solution can be computed in
polynomial time as well.

To complete the proof, observe now that any solution is provided in output when ${\tt Propagate}$ is invoked for $i=m$. After returning a tuple
of values $\tuple{a_1,...,a_m}$,  ${\tt Propagate}$  may need to backtrack to a certain index $i'\geq 1$ having some further (different) value
$a_{i'}$ to be processed, fix $X_{i'}$ with this value,
 propagate this assignment, and continue by processing variable $X_{i'+1}$.
  Thus, at most $m$ invocations of ${\tt
Propagate}$ are needed to compute the next solution, and no backtracking step may occur before we found it. Therefore, \alg\ runs WPD.  \hfill
$\lhd$
\end{proof}

By the  above theorem and the definition of domain restricted variables, the following can easily be established.

\begin{corollary}\label{cor:WPD}
Let $\boA$ be any class of $\ell$-structures such that, for each $\A\in \boA$, $(\A,\DRV(\A))$ is $\tpCovered$ through {\tt DM}. Then, for every $r$-structure $\B$, and for every set of variables $O\subseteq \DRV(\A)$, the {\em ECSP} instance $(\A,\B,O)$ is solvable WPD.
\end{corollary}

In the case of bounded arity structures and if the (hyper)tree width is the chosen decomposition method, it is not hard to see that the result
in Corollary~\ref{cor:WPD} is essentially tight. Indeed, the implication $(2)\Rightarrow(1)$ in the theorem below easily follows from the well-known dichotomy for the decision version~\cite{G07}, which is obtained in the special case of ECSP instances without output variables ($O=\emptyset$).

\begin{theorem}
Assume $FPT\neq W[1]$. Let $\boA$ be any class of $\ell$-structures of bounded arity. Then, the following are equivalent:
\begin{enumerate}
\item[(1)] $\boA$ has bounded treewdith modulo homomorphic equivalence;
\item[(2)] For every $\A\in \boA$, for every $r$-structure $\B$, and for every set of variables $O\subseteq \DRV(\A)$, the {\em ECSP} instance $(\A,\B,O)$ is solvable WPD.
\end{enumerate}
\end{theorem}

Actually, from an application perspective of this result, we observe that there is no efficient algorithm for
the no-promise problem for such classes.
In fact, the following proposition formalizes and generalizes previous observations from different authors about
the impossibility of actually trusting positive answers in the (promise) decision problem~\cite{SGG08,BDGM09}.

We say that a pair $(h,c)$ is a certified projected solution of $(\A,\B,O)$ if, by using the certificate $c$,
one may check in polynomial-time (w.r.t. the size of $(\A,\B,O)$) whether $h\in \A^\B[O]$. E.g., any full
solution extending $h$ is clearly such a certificate. If $O=\emptyset$, $h$ is also empty, and $c$ is intended
to be a certificate that  $(\A,\B)$ is a ``Yes'' instance of the decision CSP. Finally, we assume that the empty
output is always a certified answer, in that it entails that the input is a ``No'' instance, without the need
for an explicit certificate of this property.

\begin{proposition}\label{prop:NP}
The following problem is $\NP$-hard:
Given any {\em ECSP} instance $(\A,\B,O)$, compute a certified solution in $\A^\B[O]$, whenever $(\A,O)$ is $\tpCovered$ through {\tt DM}; otherwise, there are no requirements and any output is acceptable.
Hardness holds even if {\tt DM} is the treewidth method with $k=2$, the vocabulary contains just one binary relation symbol, and $O=\emptyset$.
\end{proposition}

\begin{proof}
We show a polynomial-time Turing reduction from the $\NP$-hard 3-colorability problem. Let $M$ be a Turing transducer that solves the problem,
that is,
whenever  $(\A,O)$ is $\tpCovered$ through {\tt DM},
  at the end of a computation on a given input $(\A,\B,O)$ its output tape contains a certified solution in $\A^\B[O]$, otherwise, everything is acceptable. In particular, we do not pretend that $M$ recognizes whether the above condition is fulfilled.

Then, we use $M$ as an oracle procedure within a polynomial time algorithm that solves the 3-colorability problem. Let $G=(N,E)$ be any given
graph, and assume w.l.o.g. that it contains a triangle $\{n_1,n_2\}$, $\{n_2,n_3\}$, and $\{n_3,n_1\}$. (Otherwise, select any arbitrary vertex
$n_1$ of $G$ and connect it to two fresh vertices $n_2$ and $n_3$, also connected to each other. It is easy to check that this new graph is
3-colorable if, and only if, the original graph $G$ is 3-colorable, as the two fresh vertices have no connections with the rest of the graph.)
Build the (classical) binary CSP $(\A_G,\B_{3c})$ where the vocabulary contains one relation symbol $R_E$,
 and the set of variables is $A_G=N$.
 Moreover, $R_E^{\A_G}= \{ \tuple{n_i,n_j} \mid \{n_i,n_j\}\in E\}$,
  and $R_E^{\B_{3c}}= \{\tuple{c,c'} \mid c\neq c', \{c,c'\}\subseteq \{1,2,3\}\}$.
 Consider the treewidth method for $k=2$, and compute in polynomial time the pair
  $(\A_\V,\B_\V)$ where $\A_\V={\it \ell\mbox{-}tw_2(\A_G)}$ and
  $\B_\V={\it r\mbox{-}tw_2(\B_{3c})}$. In particular, observe that the hypergraph $\HG_{\A_\V}$
  contains a hyperedge $\{n_i,n_j,n_l\}$ for every triple of vertices of $G$.

It is well-known and easy to see that $G$ is 3-colorable if, and only if, $\A_G^{\B_{3c}}\neq\emptyset$, that is, if there is a homomorphism
from $\A_G$ to a triangle (indeed, $\B_{3c}$ is a triangle). Therefore, if $G$ is 3-colorable, the triangle substructure
 $\A'$ such that $R_E^{\A'}= \{ \tuple{n_i,n_j} \mid \{i,j\}\subset \{1,2,3\}\mid i\neq j\}$ is homomorphically equivalent to $\A_G$. Moreover, in this case
  the hypergraph consisting of the single hyperedge $\{n_1,n_2,n_3\}$ is a
   tree projection of $\HG_{\A'}$ w.r.t. $\HG_{\A_\V}$, or, equivalently, the treewidth of $\A'$ is $2$.

 Now, run $M$ on input $(\A_G,\B_{3c},\emptyset)$ and consider its first output certificate $c$---say, for the sake of presentation, a full solution for the problem.
  Then check in polynomial time whether $c$ is a legal certificate---in our exemplification, whether it encodes
  a solution of the given instance.
 If this is the case, we know that $G$ is 3-colorable; otherwise, we conclude that $G$ is not 3-colorable.
  Indeed, $M$ must be correct on 3-colorable graphs, because there exists a tree projection of
  $\HG_{\A'}$ (and thus of any core of $\A_G$---note that $O=\emptyset$ and thus there is no further requirement) w.r.t. $\HG_{\A_\V}$.
  Since all these steps are feasible in polynomial-time, we are done. \hfill $\lhd$
\end{proof}

\section{Enumerating Solutions over Arbitrary Output Variables}

In this section we consider structural properties that are independent of output variables, so that tractability must hold for any desired sets
of output variables. For this case, we are able to provide certified solutions WPD, which seems the more interesting notion of tractability for
actual applications.

\begin{figure}[t]
\centering \fbox{
\parbox{0.68\textwidth}{
\begin{tabular}{l}
  \textbf{Input}: An ECSP instance $(\A,\B,O)$, where $O=\{X_1,\ldots,X_m\}$;\\
  \textbf{Output}: for each solution $h\in \A^\B[O]$, a certified solution $(h,h')$;\\
  \textbf{Method}: let $A=\{X_1,...,X_m,X_{m+1},...,X_n\}$ be the variables of $\A$;\\
  \ \ \ \ \ \ \ \ \ \ \ \ \ \ \ \ update $(\A,\B,A)$ with any of its domain restricted versions;\\
  \ \ \ \ \ \ \ \ \ \ \ \ \ \ \ \ let $\A_\V:=\lDM(\A)$, \ \ $\B_\V:=\rDM(\A,\B)$;\\
  \ \ \ \ \ \ \ \ \ \ \ \ \ \ \ \ invoke ${\tt CPropagate}$$(1,(\A_\V,\B_\V),m,\tuple{})$; \\
  \hline\\
  \vspace{-7.4mm}\\
  \hline\\
  \vspace{-7mm}\\
  \textbf{Procedure} ${\tt CPropagate}$($i$: integer, $(\A_\V,\B_\V)$: pair of structures, $m$: integer,\\
  \hspace{31mm}       $\tuple{a_1,...,a_{i-1}}$: tuple of values in $A^i$);    \\
  \textbf{begin}\\
  \textsc{1.}\ \ \ let $\B'_\V:=\GAC(\A_\V,\B_\V)$;  \\
  \textsc{2.}\ \ \ \textbf{if} $i>1$ and $\B'_\V$ is empty \textbf{then} output ``{\tt DM} failure'' and \textsc{Halt};\\
  \textsc{3.}\ \ \ let ${\it activeValues}:=\dom{(X_i)}^{\B'_\V}$;  \\
  \textsc{4.}\ \ \ \textbf{for each} element $\tuple{a_i}\in {\it activeValues}$ \textbf{do}\\
  \textsc{5.}\ \ \ $\mid$ \ \ \ \ \textbf{if} $i=n$ \textbf{then}\\
  \textsc{6.}\ \ \ $\mid$ \ \ \ \ $\mid$ \ \ \ \textbf{output} the certified solution $(\tuple{a_1,...,a_m},\tuple{a_{m+1},...,a_n})$;\\
  \textsc{7.}\ \ \ $\mid$ \ \ \ \ \textbf{else}\\
  \textsc{8.}\ \ \ $\mid$ \ \ \ \ $\mid$ \ \ \ update $\dom{(X_i)}^{\B'_\V}$ with $\{\tuple{a_i}\}$;
            \quad \quad {\em /$\ast$ $X_i$ is fixed to value $a_i$ }$\ast$/\\
  \textsc{9.}\ \ \ $\mid$ \ \ \ \ $\mid$ \ \ \ ${\tt CPropagate}$$(i+1,(\A_\V,\B'_\V),m,\tuple{a_1,...,a_{i-1},a_i})$; \\
  \textsc{10.}\ \ \hspace{-1.1mm}$\lfloor$\hspace{0.25mm} \ \ \ $\lfloor$ \ \ \ \textbf{if} $i>m$ \textbf{then} \textsc{Break};\\
  \textbf{end.}\\
\end{tabular}
}}\vspace{-2mm}
  \caption{\textbf{Algorithm} \findCertificate.} \label{fig:algoritmoCertified}\vspace{-3mm}
\end{figure}

Figure~\ref{fig:algoritmoCertified} shows the \findCertificate\ algorithm computing all solutions of an ECSP instance, with a certificate for
each of them. The algorithm is parametric w.r.t. any chosen decomposition method {\tt DM}, and resembles in its structure the \alg\ algorithm.
The main difference is that, after having found an assignment $\tuple{a_1,...,a_m}$ for the variables in $O$, \findCertificate\ still iterates
over the remaining variables in order to find a certificate for that projected solution. Of course, \findCertificate\ does not backtrack over
the possible values to be assigned to the variables in $\{X_{m+1},...,X_n\}$, since just one extension suffices to certify that this partial
solution can be extended to a full one. Thus, we break the cycle after an element $\tuple{a_i}$ is picked from its domain and correctly
propagated, for each $i>m$, so that in these cases we eventually backtrack directly to $i=m$ (to look for a new projected solution).

Note that \findCertificate\ incrementally outputs various solutions, but it halts the computation if the current $r$-structure $\B'_\V$ becomes
empty. As an important property of the algorithm, even when this abnormal exit condition occurs, we are guaranteed that all the elements
provided as output until this event are indeed solutions. Moreover, if no abnormal termination occurs, then we are guaranteed that all
solutions will actually be computed. Correctness follows easily from the same arguments used for \alg, by observing that, whenever
$(\HG_\A,\HG_{\small \lDM(\A)})$ has a tree projection, the full set of variables $A$ is $\tpCovered$ through {\tt DM}.

\begin{theorem}\label{thm:findCertificate}
Let $\A$ be an $\ell$-structure, and $O\subseteq A$ be a set of variables. Then, for every $r$-structure $\B$, \findCertificate\ computes WPD a
subset of the solutions in $\A^\B[O]$, with a certificate for each of them. Moreover,
\begin{itemize}
\item If \findCertificate\ outputs ``{\tt DM} failure'', then $(\HG_\A,\HG_{\small \lDM(\A)})$ does not have a tree projection;
    \item otherwise, \findCertificate\ computes WPD 
    $\A^\B[O]$.
\end{itemize}
\end{theorem}

Moreover, we next give some evidence that, for bounded arity classes of instances, we cannot do better than this. In particular, having bounded
width tree-decompositions of the full structure seems a necessary condition for the tractability of the enumeration problem WPD w.r.t.
arbitrary sets of output variables (and for every $r$-structure).

\begin{figure}[t]
  \centering
  \includegraphics[width=0.32\textwidth]{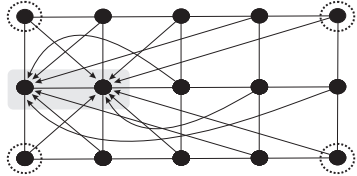}\vspace{-3mm}
  \caption{The undirected-grid structure $\A_u$, and a mapping to one of its cores.}\label{fig:grid}\vspace{-5mm}
\end{figure}

The main gadget of the proof that tree-decompositions are necessary for tractability is based on a nice feature of grids. Figure~\ref{fig:grid}
shows the basic idea for the simplest case of a relational structure $\A_u$ with only one relation symbol $R_u$ such that $R_u^{\A_u}$ is (the
edge set of) an undirected grid. Then, any substructure $\A_1$ of $\A_u$ where $R_u^{\A_1}$ contains just one tuple is a core of $\A_u$.
However, if we consider the variant of $\A_u$ where there is a domain constraint $\dom(X)$ for every corner $X$ of the grid (depicted with the
circles in the figure), then the unique core is the whole structure. In fact, we next prove that this property holds for any relational
structure whose Gaifman graph is a grid.

\vspace{-1mm}\begin{lemma}\label{lem:grid}
Let $\A$ be an $\ell$-structure whose Gaifman graph is a grid $G$. Moreover, let $O\subseteq A$ be the set of its four corners, and assume they
are domain restricted, i.e., $O\subseteq \DRV(\A)$. Then, $\A$ is a core.
\end{lemma}
\begin{proof}
Let $G$ be such a $k_1 \times k_2$ grid, and consider any homomorphism $h$ that maps $\A$ to any of its substructures $\A'$. Since the four
corners $v_{1,1},v_{1,k_2},v_{k_1,1},v_{k_1,k_2}$ are domain restricted, $h(X)=X$ must hold for each of them (as $\tuple{X}$ is the one tuple
of its domain constraint $\dom(X)^\A$). We say that such elements are fixed.

Consider the first row $r_1=(v_{1,1},v_{1,2},\ldots,v_{1,k_2})$ of $G$. We have seen that its endpoints, which are grid-corners, are fixed. It
is easy to check that $h$ cannot map the path $r_1$ to any path that is longer than $r_1$. However, $r_1$ is the shortest path connecting the
fixed endpoints  $v_{1,1}$ and $v_{1,k_2}$, and hence it must be mapped to itself. That is, $h(X)=X$ for every element $X$ occurring in $r_1$,
and thus, by the same reasoning, for every element $X$ occurring in the last row, and in the first and the last columns of the grid. It follows
that the endpoints $v_{2,1}$ and $v_{2,k_2}$ of the second row $r_2$ are fixed as well, and we may apply the same argument to show that all
elements occurring in $r_2$ are fixed, too. Eventually, row after row, we get that all elements of $A$ are fixed, and thus the identity mapping
is the only possible endomorphism for $\A$, which entails that $\A$ is a core. \hfill $\lhd$
\end{proof}

We also exploit the grid-based construction from \cite{G07}, whose properties relevant to this paper may be summarized as follows.

\begin{proposition}[\cite{G07}]\label{prop:grohe}
Let $k\geq 2$ and $K = \binom{k}{2}$, and let $\A$ be any $\tau$-structure such that the $(k \times K)$-grid is a minor of the Gaifman graph of
a core of $\A$. For any given graph $G$, one can compute in polynomial time (w.r.t. $\size{G}$) a $\tau$-structure $\B_{\A,G}$ such that $G$
contains a $k$-clique if, and only if, there is a homomorphism from $\A$ to $\B_{\A,G}$.
\end{proposition}

We can now prove the necessity of bounded treewidth for tractability WPD.

\begin{theorem}\label{thm:closed} Assume $\FPT\neq {\rm W}[1]$. Let $\boA$ be any bounded-arity recursively-enumerable class of
$\ell$-structures closed under taking minors. Then, the following are equivalent:

\vspace{-2mm}\begin{enumerate}
\item[(1)] $\boA$ has bounded treewdith;
\item[(2)] For every $\A\in \boA$, for every $r$-structure $\B$, and for every set of variables $O\subseteq A$, the {\em ECSP} instance $(\A,\B,O)$ is solvable WPD.
\end{enumerate}
\end{theorem}

\begin{proof}
The fact that $(1) \Rightarrow (2)$ holds follows by specializing Theorem~\ref{thm:findCertificate} to the tree decomposition method. We next
focus on showing that $(2) \Rightarrow (1)$ also holds.

Let $\boA$ be such a bounded-arity class of $\ell$-structures closed under taking minors, and having unbounded treewidth. From this latter
property, by the Excluded Grid Theorem~\cite{RS86} it follows that every grid is a minor of the Gaifman graph of some $\ell$-structure in
$\boA$. Moreover, because this class is closed under taking minors, every grid is actually the Gaifman graph of some $\ell$-structure in
$\boA$.

Assume there is a deterministic Turing machine $M$ that is able to solve with polynomial delay any {\em ECSP} instance $(\A,\B,O)$ such that
$\A\in\boA$. We show that this entails the existence of an $\FPT$ algorithm to solve the ${\rm W}[1]$-hard problem $p$-\textsc{Clique}, which
of course implies $\FPT={\rm W}[1]$.

Let $G$ be a graph instance of the $p$-\textsc{Clique} problem, with the fixed parameter $k\geq 2$. We have to decide whether $G$ has a clique
of cardinality $k$. We enumerate the recursively enumerable class $\boA$ until we eventually find an $\ell$-structure $\A$ whose Gaifman graph
is the $(k \times K)$-grid. Let $\tau$ be its vocabulary. Note that searching for this structure $\A$ depends on the fixed parameter $k$ only
(in particular, it is independent of $G$).

Let $O\subseteq A$ be the variables at the four corners of this grid, and let $\A'$ be the extension of $\A$ such that, for every variable
$X\in O$, the vocabulary $\tau'$ of $\A'$ contains the domain relation-symbol $\dom(X)$. Thus, the Gaifman graph of $\A'$ is the same $(k
\times K)$-grid as for $\A$, but its four corners are domain-restricted in $\A'$ ($O\subseteq \DRV(\A')$). From Lemma~\ref{lem:grid}, $\A'$ is
a core.

Recall now the grid-based construction in Proposition~\ref{prop:grohe}: We can build in polynomial time (w.r.t. $\size{G}$) a structure
$\B'_{\A',G}$ such that there is a homomorphism from $\A'$ to $\B'_{\A',G}$ if, and only if, $G$ has a clique of cardinality $k$.

Consider the ECSP instance $(\A,\B,O)$ where $\A\in\boA$ by construction, and $\B$ is the restriction of $\B'_{\A',G}$ to the vocabulary
$\tau$. Thus, compared with $\B'_{\A',G}$, the $r$-structure $\B$ may miss the domain constraint $\dom(X)$ for some output variable $X\in O$.
It is easy to see that $h$ is a homomorphism from $\A'$ to $\B'_{\A',G}$ if, and only if, $h$ is a homomorphism from $\A$ to $\B$ such that,
for every $X\in O$, $h(X)\in\dom(X)^{\B'_{\A',G}}$. Therefore, to decide whether such a homomorphism exists (and hence to solve the clique
problem), we can just enumerate WPD the set of solutions $\A^\B[O]$ and check whether the four domain constraints on the corners of $\A'$ are
satisfied by any of these solutions. Now, recall that $\B'_{\A',G}$ is built in polynomial time from $G$, and thus every variable may take only
a polynomial number of values, and of course all combinations of four values from $\dom(X)^{\B'_{\A',G}}$, $X\in O$, are polynomially many. It
follows that $M$ actually takes polynomial time for computing $\A^\B[O]$, and one may then check in polynomial time whether the additional
domain constraints in $\A'$ are satisfied or not by some solution in $\A^\B[O]$.

By combining the above ingredients, we got an $\FPT$ algorithm to decide whether $G$ has a clique of cardinality $k$.~\hfill~$\lhd$
\end{proof}


\begin{thebibliography}{10}

\bibitem{adler08} I. Adler.
\newblock Tree-Related Widths of Graphs and Hypergraphs.
\newblock {\em SIAM Journal Discrete Mathematics}, 22(1), pp. 102--123, 2008.

\bibitem{ABD07} A. Atserias, A. Bulatov, and V. Dalmau.
\newblock On the Power of k-Consistency,
\newblock In {\em Proc. of ICALP'07}, pp. 279--290, 2007.

\bibitem{bern-good-81} P.A.~Bernstein and N.~Goodman.
\newblock The power of natural semijoins.
\newblock {\em SIAM Journal on Computing}, 10(4), pp. 751--771, 1981.

\bibitem{BDGM09} A. Bulatov, V. Dalmau, M. Grohe, and D. Marx.
Enumerating Homomorphism. In \emph{Proc. of 
STACS'09}, pp. 231--242, 2009.

\bibitem{CD05} H. Chen and V. Dalmau.
\newblock Beyond Hypertree Width: Decomposition Methods Without Decompositions.
\newblock In {\em Proc. of CP'05}, pp. 167--181, 2005.

\bibitem{cjg-08} D. Cohen, P. Jeavons, and M. Gyssens. A unified theory of structural tractability for constraint satisfaction problems.
    \emph{Journal of Computer and System Sciences}, 74(5): 721-743, 2008.

\bibitem{C04} D.A. Cohen. Tractable Decision for a Constraint Language Implies Tractable Search. \emph{Constraints}, 9(3), 219--229, 2004.

\bibitem{DI92} R. Dechter and A. Itai. Finding All Solutions if You can Find One. In Proc. of \emph{AAAI-92 Workshop on Tractable Reasoning},
    pp. 35--39, 1992.

\bibitem{down-fell-99}
R.G. Downey and M.R. Fellows.
\newblock {\em Parameterized Complexity}.
Springer, New York, 1999.

\bibitem{GS84} N. Goodman and O. Shmueli.
\newblock The tree projection theorem and relational query processing.
\newblock {\em Journal of Computer and System Sciences}, 29(3), pp. 767--786, 1984.

\bibitem{gott-etal-00} G.~Gottlob, N.~Leone, and F.~Scarcello. A Comparison of Structural CSP Decomposition Methods. \emph{Artificial
    Intelligence}, 124(2): 243--282, 2000.

\bibitem{gott-etal-99} G.~Gottlob, N.~Leone, and F.~Scarcello.
\newblock Hypertree decompositions and tractable queries.
\newblock {\em Journal of Computer and System Sciences}, 64(3), pp. 579--627, 2002.

\bibitem{gott-etal-03} G.~Gottlob, N.~Leone, and F.~Scarcello.
\newblock Robbers, marshals, and guards: game theoretic and logical characterizations of hypertree width.
\newblock {\em J. of Computer and System Sciences}, 66(4), pp. 775--808, 2003.

\bibitem{GMS07} G. Gottlob, Z. Mikl\'os, and T. Schwentick.
\newblock Generalized hypertree decompositions: $\NP$-hardness and tractable variants.
\newblock {\em Journal of the ACM}, 56(6), 2009.

\bibitem{GS10} G. Greco and F. Scarcello.
\newblock The Power of Tree Projections: Local Consistency, Greedy Algorithms, and Larger Islands of Tractability.
\newblock To appear in \emph{Proc. of 
PODS'10}. 

\bibitem{GSS01} M.~Grohe, T. Schwentick, and L. Segoufin.
\newblock When is the evaluation of conjunctive queries tractable?
\newblock In \emph{Proc. of 
STOC'01}, pp. 657--666, 2001.

\bibitem{G07} M.~Grohe.
\newblock The complexity of homomorphism and constraint satisfaction problems seen from the other side.
\newblock {\em Journal of the ACM}, 54(1), 2007.

\bibitem{grohe-marx-06} M. Grohe and D. Marx. Constraint solving via fractional edge covers. In \emph{Proc. of
SODA'06}, pp. 289--298, 2006.

\bibitem{M09} D. Marx.
\newblock Approximating fractional hypertree width.
\newblock In {\em Proc. of SODA'09}, pp. 902--911, 2008.

\bibitem{M10} D. Marx.
Tractable Hypergraph Properties for Constraint Satisfaction and Conjunctive Queries. To appear in \emph{Proc. of
STOC'10}.

\bibitem{RS84} N. Robertson and P.D. Seymour.
\newblock Graph minors III: Planar tree-width.
\newblock {\em Journal of Combinatorial Theory, Series B}, 36, pp. 49-–64, 1984.

\bibitem{RS86} N. Robertson and P.D. Seymour.
\newblock Graph minors V: Excluding a planar graph.
\newblock {\em Journal of Combinatorial Theory, Series B}, 41, pp. 92-–114, 1986.

\bibitem{SS93} Y. Sagiv and O Shmueli.
\newblock Solving Queries by Tree Projections.
\newblock {\em ACM Transaction on Database Systems}, 18(3), pp. 487--511, 1993.

\bibitem{SGG08}
F. Scarcello, G. Gottlob, and G. Greco.
\newblock Uniform Constraint Satisfaction Problems and Database Theory.
\newblock In {\em Complexity of Constraints}, LNCS 5250, pp. 156--195,
Springer-Verlag, 2008.

\end{thebibliography}
\end{document}